\documentclass[11pt]{article}
\usepackage{fullpage}

\usepackage{microtype}
\usepackage{graphicx}
\usepackage{booktabs} 

\usepackage{amsthm}
\usepackage{amsmath,amssymb,amsfonts}
\usepackage{algorithm, algorithmic}
\usepackage{graphicx}
\usepackage{textcomp}
\usepackage{subfig}
\usepackage{xcolor}
\usepackage{tikz}
\usetikzlibrary{arrows,decorations.markings, positioning, calc}
\def\BibTeX{{\rm B\kern-.05em{\sc i\kern-.025em b}\kern-.08em
    T\kern-.1667em\lower.7ex\hbox{E}\kern-.125emX}}
\usepackage{enumitem}
\usepackage{hyperref}
\usepackage{cleveref}
\usepackage{forloop}

\usepackage{pgfplots}
\pgfplotsset{compat=1.12}

\usepackage{blindtext}

\newtheorem{theorem}{Theorem}
\newtheorem{lemma}[theorem]{Lemma}

\newtheorem{remark}[]{Remark}

\newcommand{\Ocal}{\mathcal{O}}

\newcommand{\Rbb}{\mathbb{R}}
\newcommand{\Qbb}{\mathbb{Q}}
\newcommand{\Zbb}{\mathbb{Z}}
\newcommand{\Nbb}{\mathbb{N}}

\newcommand{\BNN}{\text{\scriptsize BNN}}  

\DeclareMathOperator{\diam}{diam}

\newcommand{\eps}{\varepsilon}

\newcommand{\norm}[1]{\left\lVert#1\right\rVert}

\definecolor{TUgreen}{HTML}{649600}

\newdimen\XCoord
\newdimen\YCoord
\newcommand*{\ExtractCoordinate}[1]{\path (#1); \pgfgetlastxy{\XCoord}{\YCoord};}%

\hypersetup{final}

\begin{document}

\title{Universal Approximation Theorems of Fully Connected Binarized Neural Networks}
\author{Mikail Yayla, Mario G\"unzel, Burim Ramosaj, and Jian-Jia Chen\thanks{ mikail.yayla@tu-dortmund.de, mario.guenzel@tu-dortmund.de, burim.ramosaj@tu-dortmund.de, jian-jia.chen@tu-dortmund.de}\\
TU Dortmund University, Dortmund, Germany
}

\maketitle

\begin{abstract}
Neural networks (NNs) are known for their high predictive accuracy in complex learning problems. 
Beside practical advantages, NNs also indicate favourable theoretical properties such as universal approximation (UA) theorems. 
Binarized Neural Networks (BNNs) significantly reduce time and memory demands by restricting the weight and activation domains to two values. Despite the practical advantages, theoretical guarantees based on UA theorems of BNNs are rather sparse in the literature. We close this gap by providing UA theorems for fully connected BNNs under the following scenarios: (1) for binarized inputs, UA can be constructively achieved under one hidden layer; (2) for inputs with real numbers, UA can not be achieved under one hidden layer but can be constructively achieved under two hidden layers for Lipschitz-continuous functions. Our results indicate that fully connected BNNs can approximate functions universally, under certain conditions. 
\end{abstract}

\section{Introduction}
\label{sec:intro}

In several applications such as image or voice recognition, Neural Networks (NN) are key models for training a machine in supervised learning problems, see e.g. \cite{krizhevsky2012imagenet}, \cite{simonyan2014very}, \cite{szegedy2015going} or \cite{rajkomar2018scalable}. The usage of these learners has experienced an increasing trend due to their high predictive accuracy in complex learning problems, while computational capacities have increased making their usage applicable. 

However, the high accuracy of NNs is achieved by deep structures with many layers and a massive number of parameters. This leads to bottlenecks in data movement and multiply accumulate (MAC) operations in general-purpose computing systems. Many studies have therefore focused on methods to improve the energy and inference latency of NNs. Algorithmic approaches employ pruning of weights, neurons, or larger structures to achieve efficient NN models~\cite{choudhary/etal/2020}. On the hardware side, the design of various specialized devices to accelerate NN operations for energy efficiency and fast application has been explored~\cite{yiranchen/etal/2020}.

Another approach for reducing computational demands is quantization resp.\@ binarization in NNs. This results in the restriction of weights in the network to two or three possible outcomes usually lying in the range of $\{\pm 1\} $, see e.g.~\cite{hubara/etal/2016},~\cite{hirtzlin/etal/2019stoch} or~\cite{sari/etal/2019}. The benefits are multifold: while time complexity is significantly reduced through the simpler network structures, space complexity is also kept at a comparably low level. High predictive accuracy of binarized NNs in regression or classification and online training have been experimented in~\cite{hodge2014short} or~\cite{yu2016binary}, for example. 

While most of the results in the literature for quantized NNs have been focused on experimental research in empirical and simulation based analyzes, theoretical guarantees such as universal approximation (UA) theorems are rather sparse. In~\cite{ding/etal/2019} or~\cite{spallanzani/etal/2019}, for example, UA theorems for quantized NNs have been established under assumptions such as Lipschitz-continuity of the approximating function, while not fully connected networks have been considered. Therefore, the choice of the corresponding network structure remains mostly unclear. 
We aim to close this gap by providing minimal layer sizes required for attaining UA theorems in fully connected binarized NNs (FC-BNNs) instead. 

\textbf{Our contributions} are the fundamental results of existence or non-existence of UA theorems for FC-BNNs
under the following scenarios: 
\begin{itemize}
\item We first start with FC-BNNs with binary inputs and potentially real-valued function image. In Section \ref{sec:1_fc_bin}, we establish an UA theorem for this scenario. 

\item Extending the feature domain to real-values while focusing on single hidden layer FC-BNNs, we construct a counterexample indicating the violation of UA-like theorems in Section \ref{sec:1_fc_real}. Therefore, we can show that a single hidden layer FC-BNN is not capable of universally approximating functions in case of non-binary inputs.  In order to obtain an UA theorem for this case, the inclusion of an extra hidden layer is required. Hence, in Section \ref{sec:2_fc_real}, we establish the UA theorem for FC-BNNs with two hidden layers and real-valued inputs. 

\item While the weights directing to the output layer of a quantized or binarized NN are real values, we show in Section~\ref{sec:removal-real-weights} that the restriction to quantized or binarized weight values for the output layers do not have severe effects. Hence, we theoretically guarantee that any quantized or binarized NN with real valued output weights can be transformed into a NN with fully quantized resp.\@ binarized weights keeping output accuracy at a high level. 
\end{itemize}

Our results indicate that BNNs show beside efficient time and space complexity properties also UA properties with minimal layer-size requirements.
In addition to the above technical contributions from Sections~\ref{sec:1_fc_bin} to~\ref{sec:removal-real-weights}, Section~\ref{sec:relatedwork} provides a short summary of the related work, Section~\ref{sec:model} defines the system model studied in this paper, and Section~\ref{sec:conclusion} concludes this paper.

\section{Related Work}\label{sec:relatedwork}

Recent research on quantized NNs (QNNs) and BNNs have been focusing on two main fields: (i) algorithmic implementation of these networks based on variants of the Backpropagation algorithm in (stochastic) decent algorithms and  (ii) theoretical guarantees regarding algorithmic convergence and universal approximations. Regarding the first point, the work of \cite{hubara2017quantized}, for example, delivers methodologies how to compress neural network structures in order to train QNNs with weights of extremely low precision (1 bit storage). Therein, the AlexNet could be quantized while almost maintaining predictive accuracy. In \cite{choi2016towards}, designs for network structures are proposed that minimize performance loss. The latter has been quantized using Hessian weighted distortion measures. Theoretical results in this direction have been analyzed in \cite{li2017training}, where several quantization methods in Convolutional NNs using the gradient descent method have been investigated for convergence purposes

The first UA theorems for NNs date back to the work of \cite{cybenko1989approximation} using sigmoid activation functions and an arbitrary layer width. Modifications and extensions have been established in \cite{hornik1991approximation} or  \cite{barron1994approximation}, for example. UA properties to quantized NNs have been considered in \cite{spallanzani/etal/2019}, \cite{ding/etal/2018} and \cite{wang/etal/2018}, for example. The first work focused on the uniform approximation capabilities of QNNs for Lipschitz-continuous functions. Furthermore, the authors reveal that the same set of functions that can be approximated by a Deep NN can also be approximated by QNNs. The results are mainly based on not fully connected QNNs and leave room for scenarios, in which the layers of a QNN are fully connected. 
We aim to close this gap by considering FC-BNNs with moderate numbers of layers required to obtain UA properties, while also extending the work to potential binarized inputs as well. 

Differently to that, the authors in \cite{wang/etal/2018} focused on stochastic-computing based NNs and show UA properties in almost sure sense. The work of \cite{ding/etal/2018} considers UA properties of quantized ReLU networks for locally integratable functions on the Sobolev space. Furthermore, they provide tight upper bounds on the number of weights and the memory size of quantized ReLU networks. While the theoretical results reveal UA properties for either a given number of weights in the network or a given weight maximum precision, fully connected, minimal layer properties for obtaining UA theorems are not directly deducible. Therefore, we aim to close this gap by providing minimal layer requirements while focusing on FC-BNNs, for which we derive UA properties. We focus on fully connected networks, since this is the intuitively most trivial network structure when constructing moderate layer-sized BNNs. 

\section{System Model and Preliminary Results}
\label{sec:model}

\begin{figure}
    \centering
    \newcommand{\circlediam}{0.3cm}
    \begin{tikzpicture}[xscale=0.85]
        \node[minimum size=0.5cm,draw,circle] (i1) at (-3,3) {$x_1$};
        \node[minimum size=0.5cm,draw,circle] (i10) at (-3,1) {$x_2$};
        \node[minimum size=0.5cm,draw,circle] (i2) at (-3,-1) {$x_d$};
        
        \node[minimum size=1cm,draw,circle] (h3) at (0,3) {$o^1_1$};
        \node[minimum size=1cm,draw,circle] (h2) at (0,1) {$o^1_2$};
        \node[minimum size=1cm,draw,circle] (h5) at (0,-1) {$o^1_{n_1}$};
        
        \node[anchor=north, text width=1.5cm] at (-2.25,-1.5) {input layer\\($\ell = 0$)};
        \node[anchor=north, text width=2cm] at (0.75,-1.5) {hidden layer\\($\ell = 1$)};
        \node[anchor=north, text width=2cm] at (3.75,-1.5) {output layer\\($\ell = 2$)};
        
        
        \node at ($(i10)!.5!(i2)$) {\vdots};
        \node at ($(h2)!.5!(h5)$) {\vdots};
        
        \coordinate (D) at ($(h3)!.5!(h5)$);
        \ExtractCoordinate{D};
        \node[minimum width=1cm,draw,circle] (o1) at (3,\YCoord) {$\sum$};
        
        \node[] (out) at (5,\YCoord) {$\mathcal{O}^{h=1}_{\BNN}(\mathbf{x})$};
        
        
        \draw[->] (i1) -- (h2);
        \draw[->] (i2) -- (h2);
        \draw[->] (i1) -- node[above] {$w^1_{i,j}$} (h3);
        \draw[->] (i2) -- (h3);
        \draw[->] (i1) -- (h5);
        \draw[->] (i2) -- (h5);
        
        \draw[->] (i10) -- (h3);
        \draw[->] (i10) -- (h2);
        \draw[->] (i10) -- (h5);
        
        \draw[->] (h2) -- (o1);
        \draw[->] (h3) -- node[above] {$w^2_{j}$} (o1);
        \draw[->] (h5) -- (o1);
        
        \draw[->] (o1) -- (out);
    \end{tikzpicture}
    \caption{Single hidden layer BNN-Model}
    \label{fig:bnn_model}
\end{figure}
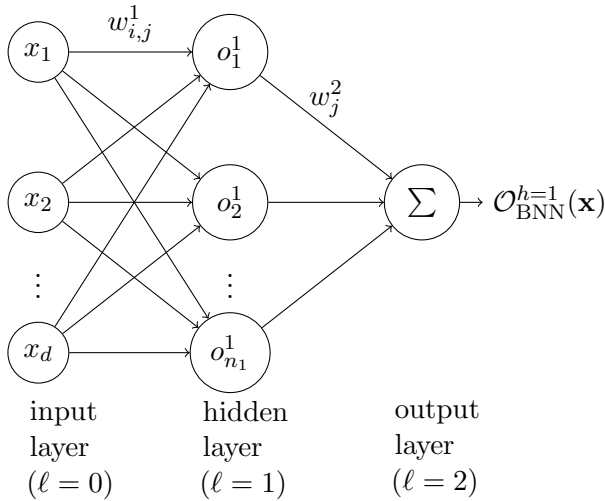

\begin{figure}
    \centering
    \newcommand{\circlediam}{0.3cm}
    \scalebox{0.85}{\begin{tikzpicture}[xscale=0.85]
        \node[minimum size=0.5cm,draw,circle] (i1) at (-3,3) {$x_1$};
        \node[minimum size=0.5cm,draw,circle] (i10) at (-3,1) {$x_2$};
        \node[minimum size=0.5cm,draw,circle] (i2) at (-3,-1) {$x_d$};
        
        \node[minimum size=1cm,draw,circle] (h3) at (0,3) {$o^1_1$};
        \node[minimum size=1cm,draw,circle] (h2) at (0,1) {$o^1_2$};
        \node[minimum size=1cm,draw,circle] (h5) at (0,-1) {$o^1_{n_1}$};
        
        \node[minimum size=1cm,draw,circle] (h23) at (3,3) {$o^2_1$};
        \node[minimum size=1cm,draw,circle] (h22) at (3,1) {$o^2_2$};
        \node[minimum size=1cm,draw,circle] (h25) at (3,-1) {$o^2_{n_2}$};
        
        \node at ($(i10)!.5!(i2)$) {\vdots};
        \node at ($(h2)!.5!(h5)$) {\vdots};
        \node at ($(h22)!.5!(h25)$) {\vdots};
        
        \coordinate (D) at ($(h3)!.5!(h5)$);
        \ExtractCoordinate{D};
        \node[minimum width=1cm,draw,circle] (o1) at (5,\YCoord) {$\sum$};
        
        \node[] (out) at (7,\YCoord) {$\mathcal{O}^{h=2}_{\BNN}(\mathbf{x})$};
        
        
        \draw[->] (i1) -- (h2);
        \draw[->] (i2) -- (h2);
        \draw[->] (i1) -- node[above] {$w^1_{i,j}$} (h3);
        \draw[->] (i2) -- (h3);
        \draw[->] (i1) -- (h5);
        \draw[->] (i2) -- (h5);
        
        \draw[->] (i10) -- (h3);
        \draw[->] (i10) -- (h2);
        \draw[->] (i10) -- (h5);
        
        \draw[->] (h3) -- node[above] {$w^2_{j,k}$} (h23);
        \draw[->] (h2) -- (h22);
        \draw[->] (h5) -- (h25);
        
        \draw[->] (h3) -- (h25);
        \draw[->] (h2) -- (h23);
        \draw[->] (h5) -- (h22);
        
        \draw[->] (h3) -- (h22);
        \draw[->] (h2) -- (h25);
        \draw[->] (h5) -- (h23);
        
        \draw[->] (h22) -- (o1);
        \draw[->] (h23) -- node[above] {$w^3_{k}$} (o1);
        \draw[->] (h25) -- (o1);
        
        \draw[->] (o1) -- (out);
        
        \node[anchor=north, text width=1.5cm] at (-2.5,-1.5) {input layer\\($\ell = 0$)};
        \node[anchor=north, text width=2cm] at (0.75,-1.5) {hidden layer\\($\ell = 1$)};
        \node[anchor=north, text width=2cm] at (3.75,-1.5) {hidden layer\\($\ell = 2$)};
        
        \node[anchor=north, text width=2cm] at (6.75,-1.5) {output layer\\($\ell = 3$)};
        
    \end{tikzpicture}}
    \caption{Two hidden layers BNN-Model.}
    \label{fig:bnn_model2}
\end{figure}
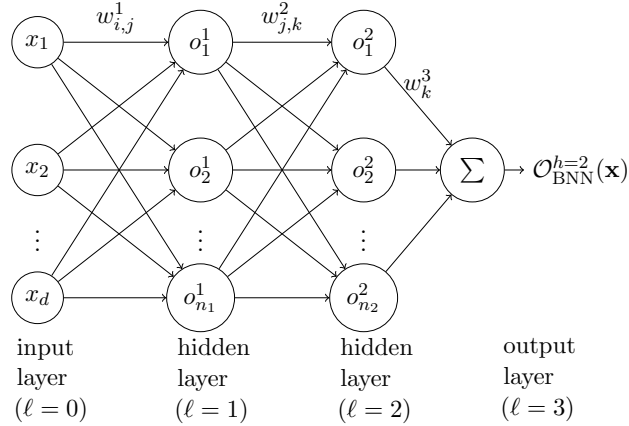

\subsection{System Model and NN Structures}
We consider fully connected binarized neural networks (BNNs) as depicted in Figure~\ref{fig:bnn_model} with one hidden layer (${h=1}$), denoted as $\Ocal^{h=1}_{\BNN}$, and Figure~\ref{fig:bnn_model2} with two hidden layers (${h=2}$), denoted as $\Ocal^{h=2}_{\BNN}$.

Throughout the paper, we denote vectors and matrices in \textbf{bold}, while scalars are printed in usual format. Furthermore, superscripts in $\phi, w, b$ and $o$ representing different parts of a network have not be understood in mathematical exponent sense, but rather as index values referring to different parts of the network.
The superscript $()^\top$ denotes the transposed vector/matrix.

Let $\boldsymbol{x} = (x_1, \dots, x_d)^\top \in X \subseteq \mathbb{R}^{d}$ be the input to the BNN, where $x_i$ is the $i$-th input value of $\boldsymbol{x}$, with $1 \le i \le d$. We consider the scenarios where $x_i$ is either a \emph{real} number in $\Rbb$ or a \emph{binary} number in $\{\pm1\}$ for all $i \in \{1, \dots, d\}$ as specified in each section. The input $\boldsymbol{x}$, which we denote as $\boldsymbol{\phi}^{0}$ in the following, is fed into the first hidden layer, i.e., the one with layer index $\ell=1$, of the BNN. This layer subsequently computes its output, i.e., the output of the first hidden layer, denoted as $\boldsymbol{\phi}^{1}$.

Furthermore, denote by $L \in \mathbb{N}$ the total number of layers, such that the network consists of $L-1$ hidden layers and one output layer. Each layer consists of $n_\ell \in \mathbb{N}$ neurons. More specifically, the $i$-th neuron in the  \mbox{$\ell$-th} hidden layer, with $1\leq i \leq n_{\ell}$ and $1\leq\ell\leq L-1$, has input $\boldsymbol{\phi}^{\ell-1} =  (\phi_{1}^{\ell-1}, \dots, \phi_{n_{\ell-1}}^{\ell-1})^\top$ and output $o_i^\ell(\boldsymbol{\phi}^{\ell-1})$, which is used as component $\phi_{i}^{\ell-1}$ in the input $\boldsymbol{\phi}^{\ell}$ 
for the neurons of the subsequent layer.
The output of the last layer is denoted by $out(\boldsymbol{\phi}^{L-1})$.
For the hidden layers, the contribution of the neuron $i$ in the $(\ell-1)$-th layer to neuron $j$ in the $\ell$-th hidden layer is evaluated using binarized weight $w^\ell_{i,j} \in \{\pm1\}$. By iterating over weights and inputs for the convolution, the pre-activation value $s^\ell_j(\boldsymbol{\phi}^{\ell-1})$ of the neuron $j$ in the $\ell$-th layer is  $s^\ell_j(\boldsymbol{\phi}^{\ell-1}) = \sum\limits_{t=1}^{n_{\ell-1}} w^\ell_{t,j} \phi^{\ell-1}_{t}$.
The neurons compute their output by $o^{\ell}(\boldsymbol{\phi}^{\ell-1}) := \sigma_{b^\ell_j} (s^\ell_j(\boldsymbol{\phi}^{\ell-1}))$, with the activation function
\begin{equation}\label{eq:activation_fct}
	\sigma_{b^\ell_j}: \mathbb{R} \to \{\pm1\},~t \mapsto \begin{cases}
	1 & ,\,t > b^\ell_j \\
	-1 & ,\,\text{else}
	\end{cases}
\end{equation}
where 
$b^\ell_j \in \mathbb{R}$ is a threshold specifying parameter, which can be calculated from batch normalization parameters (see~\cite{sari/etal/2019} for details on batch normalization in BNNs).
When $\ell=L$, i.e., when we consider the output layer, the final outcome of the network is given by $out(\boldsymbol{\phi}^{L-1})$ and is computed by $\sum\limits_{i=1}^{n_{L-1}} w^{L}_{i} \phi^{L-1}_{i}$ with real values $w^{L}_{i} \in \mathbb{R}$. 
We note that the weights of the output layer are real values to allow the BNN to output real values; otherwise the BNN would be limited to output integer values. The weights of all other layers are restricted to binary values only. NN with multiple outcomes such as in multivariate response regression do consider multiple output neurons in layer $L$. In this work, we consider NNs with one output neuron.

For the BNNs in Figures~\ref{fig:bnn_model} and~\ref{fig:bnn_model2}, the output of a single hidden layer BNN at input $\boldsymbol{x}$ is defined by $\Ocal^{h=1}_{BNN}(\boldsymbol{x}) := out ( o^1 (\boldsymbol{x}) )$ and the output of a binarized neural network at $\boldsymbol{x}$ with two hidden layers is computed by $\Ocal^{h=2}_{BNN}(\boldsymbol{x}) := out ( o^2(o^1 (\boldsymbol{x}) ))$.

In this work, for simplicity of presentation, we consider binarization of weights and activations, i.e., $w_{i,j}^\ell$ and each component of $\boldsymbol{\phi}^{\ell}$ after applying the activation function in $\eqref{eq:activation_fct}$ are in $\{\pm1\}$ for $1 \leq \ell \leq L-1$ as presented above. In practice, the binarization to $\{0,1\}$ is used, for more efficient execution of the multiplication and accumulation (MAC) operation of neurons. In this case, the multiplication, summation and activation can be computed with
\begin{equation}\label{MACoperator}
    2*{popcount}(XNOR(\boldsymbol{W}^{\ell-1}_j, \boldsymbol{\phi}^{\ell-1})) - n_{\ell-1} > b^{\ell}_j,
\end{equation}
where $\boldsymbol{W}^{\ell-1}_j = (w^{\ell-1}_{1,j}, \dots, w^{\ell-1}_{n_{\ell-1},j})^\top \in \{0,1\}^{n_{\ell-1}}$ denotes the $n_{\ell - 1}$-tuple of the incoming weights to the $j$-th neuron in the $\ell$-th layer, $\boldsymbol{\phi}^{\ell-1} \in \{0,1\}^{n_{\ell-1}}$ the binarized input to neuron $j$ located at layer $\ell$, $popcount$ counts the number of ones (hamming weight), $n_{\ell-1}$ is the number of bits in the XNOR operands, and $b^\ell_j$ is the threshold specifying parameter. The result of this comparison is a binary value~\cite{hubara/etal/2016},~\cite{sari/etal/2019}. Therefore, the operations such as summation, multiplication in $s^\ell_j(\boldsymbol{\phi}^{\ell-1})$ and the usage of the activation function in $\eqref{eq:activation_fct}$ have to be understood 
in MAC operation sense as introduced in $\eqref{MACoperator}$. 
The usage of the proposed operations in $\eqref{MACoperator}$ is justifiable due to the properties of the binary XNOR operation under binarized inputs to the neurons. If the inputs are binarized (as explored in~\cite{hirtzlin/etal/2019stoch}), then the MAC operations of the first layer can also be computed with $\eqref{MACoperator}$. If the input is in $\Rbb$, then the usual MAC operations are performed.

\subsection{Preliminary Results of Universal Approximation}\label{sec:uaptheorems}

The development of UA properties of neural networks date back to the work of \cite{cybenko1989approximation} and were one of the first theoretical breakthroughs for neural networks. In the sequel, we formulate the UA theorem for neural networks with weights in $\mathbb{R}$ and \textit{sigmoid}-type activation functions, which are non-decreasing functions with the property that $\sigma(x) \rightarrow -1$, as $x \downarrow - \infty$ and $\sigma(x) \rightarrow 1$, as $x \uparrow \infty$. Note that the activation function in \eqref{eq:activation_fct} considered in our paper is of sigmoid-type. The upcoming theorem is leaned on \cite{devroye2013probabilistic}, page 519. We first focus on neural networks with $h = 1$ hidden layers with inputs lying on a hyper-rectangle and sigmoidal-type activation functions. Recalling Theorem $2$ in \cite{hornik1991approximation}, one can relax the sigmoidal-type condition and instead assume that the activation function is bounded, non-constant and continuous, while the function support is compact. Even stronger results regarding the class of activation functions leading to UA properties of a neural network can be found in \cite{leshno1993multilayer}. Therein, all non-polynomial activations can lead to UA properties of neural networks. Both works indicate the importance of the network structure rather than the choice of the activation function. Note that function denseness in \cite{hornik1991approximation} is defined on $L^p$-space for $1 \le p < \infty$ and is slightly different to our deterministic function approximation capabilities.

\begin{theorem}[Universal Approximation Theorem for NN]
    For every continuous function $f: [a,b]^d \rightarrow \Rbb$ and for every $\epsilon > 0$, there exists a neural network with one hidden layer and sigmoid-type activation function denoted as $\Ocal_{N}^{h = 1}$ with the usual MAC operations such that 
    \begin{equation}\label{eq:ua_thm1}
        \sup_{\boldsymbol{x} \in [a,b]^d} \left|{\mathcal{O}_{N}^{h = 1}(\boldsymbol{x}) - f(\boldsymbol{x})} \right|< \eps
    \end{equation}{}
    holds.
\end{theorem}

UA properties can also be established for two-hidden layer neural networks. The difference to the one-hidden layer case is the distribution of the nodes in each layer. According to \cite{devroye2013probabilistic}, on page 517, little is gained from a theoretical perspective by the inclusion of an additional layer.  

\section{Single Hidden Layer, FC-BNN, Binary Inputs}
\label{sec:1_fc_bin}

In this section, we deal with the scenario of binary input values for a FC-BNN while restricting to the one-hidden layer case. Note that -- if not otherwise stated -- the assumptions in Section \ref{sec:model} regarding the MAC operations and the used activation function in $\eqref{eq:activation_fct}$ are valid. Under this scenario, we show that one-hidden layer FC-BNNs are capable to universally approximate any function living on the same domain as the input values. This is established in the upcoming Theorem \ref{thm:bin_input}. In its core, it is based on a deterministic function-dependent FC-BNN fitting on the underlying domain grid and intends to establish the first UA property of FC-BNNs in this paper. 

\begin{theorem}\label{thm:bin_input}
    Let $X=\{-1,1\}^d$ be some $d$-dimensional binary input space and let $f: X \to \Rbb$ be the ground-truth function. Assume activation functions of the type as in $\eqref{eq:activation_fct}$ together with the MAC operations given in $(\ref{MACoperator})$. Then
    for all $\eps>0$ there exists some fully connected BNN with one hidden layer denoted as $\Ocal_{\BNN}^{h=1}$ such that
    \begin{equation}\label{eq:ua_thm2}
        \sup_{\boldsymbol{x} \in X} \left|\mathcal{O}^{h=1}_{\BNN}(\boldsymbol{x}) - f(\boldsymbol{x})\right| < \eps
    \end{equation}{}
    holds. 
\end{theorem}

\begin{proof}
    First, we consider the case $d\geq 2$, afterwards we also handle the case $d=1$.
    As the input is $d$ dimensional, there are $2^d$ different inputs. We call them $\boldsymbol{x}^1, \dots, \boldsymbol{x}^{2^d}$.
    We construct the neural network by putting $n_1 = 2^d$ neurons on the first layer.
    The weights $w^1_{1,j}, \dots, w^1_{d,j}$ corresponding to the input-layer for the $j$-th neuron are set to the $j$-th possible input, i.e., $(w^1_{1,j}, \dots, w^1_{d,j})^\top = \boldsymbol{x}^j$.
    The activation threshold is set to $d-1$.
    As a result, a neuron on the first layer only has an output of $1$ if its weights coincide with the input.
    Hence, for the $j$-th input $\boldsymbol{x}^j$, only the $j$-th neuron has an output of $1$, all the others have the output $-1$.
    We choose the weights $w^2_{j} \in \mathbb{R}$ from the $j$-th neuron on the hidden layer to the output layer $(L = 2)$ such that for each $j$
    \begin{align}\label{eq:proof}
        w^2_{j} + \sum_{i \neq j} -w^2_{i} = f(\boldsymbol{x}^j)
    \end{align}
    holds. 
    Note that the above equation can be rewritten into a linear equation of the form $\boldsymbol{A}\boldsymbol{w} = \boldsymbol{b}$, where $\boldsymbol{A} = - \boldsymbol{1}_{n_1}\boldsymbol{1}_{n_1}^\top + 2 \boldsymbol{I}_{n_1} \in \{\pm1\}^{n_1 \times n_1}$ is the matrix with $1$ on the diagonal and $-1$ at all other entries, $\boldsymbol{w} = (w^2_{1} , \dots, w^2_{n_1} )^\top$, and $\boldsymbol{b} = (f(\boldsymbol{x}^1), \dots, f(\boldsymbol{x}^d))^\top$. Here, $\boldsymbol{1}_{n_1} = (1, \dots,1)^\top$ is the $n_1$-dimensional vector consisting only of ones and $\boldsymbol{I}_{n_1} \in \Rbb^{n_1 \times n_1}$ denotes the identity matrix. 
    Using Sylvester's determinant theorem, we obtain 
    \begin{align*}
        \det(\boldsymbol{A}) &= \det(2\boldsymbol{I}_{n_1} - \boldsymbol{1}_{n_1}\boldsymbol{1}_{n_1}^\top) \\
        &= \det(2I_{n_1})\cdot \det(1 - \boldsymbol{1}_{n_1}^\top (2\boldsymbol{I}_{n_1})^{-1} \boldsymbol{1}_{n_1} ) \\
        &= 2^{n_1} \cdot (1- (1/2)n_1)<0,
    \end{align*}
    since $n_1 = 2^d > 2$, for $d \ge 2$. Therefore, we can immediately deduce that the solution to $\eqref{eq:proof}$ is uniquely given by $\boldsymbol{w} = \boldsymbol{A}^{-1} \boldsymbol{b}$ if $d\geq 2$.  Hence, this BNN assigns to each input $\boldsymbol{x}^j$ the output $\Ocal(\boldsymbol{x}^j) = f(\boldsymbol{x}^j)$ and mimics the function $f$ exactly.
    
    In case of $d=1$, $\boldsymbol{A}$ as defined above does not have full rank anymore and the existence of the inverse is not guaranteed. Therefore, the proposed network structure for $d \ge 2$ is not applicable for $d =1$. We need to introduce one additional neuron on the hidden layer to deal with this case.
    The additional third neuron is defined by any weight assignment and threshold $-d-1$, i.e., its output is always $1$.
    To find the approximating neural network, we then need to assign weights $w^2_j,~j=1,2,3$ for the output layer, such that the equations 
    \begin{align}
        w^2_1 - w^2_2 + w^2_3 & = f(x^1)
        \\ -w^2_1 + w^2_2 + w^2_3 & = f(x^2)
    \end{align}
    hold.
    We leave it to the reader to check that the assignment $w^2_1:=\frac{f(x^1)-f(x^2)}{4}$, $w^2_2:=-\frac{f(x^1)-f(x^2)}{4}$ and $w^2_3:=\frac{f(x^1)+f(x^2)}{2}$ fulfills the properties from the above two equations.
\end{proof}

This theorem proves not only an universal approximation theorem, but also the equality between functions on a binary input space and FC-BNNs with one hidden layer on a binary input space. Extending the input space to an arbitrary, countable and finite space is not directly possible, since counterexamples can be constructed for such scenarios. This will be partly seen in the next section.

\section{Negative Result: Single Hidden Layer, FC-BNN, Real Inputs}
\label{sec:1_fc_real}

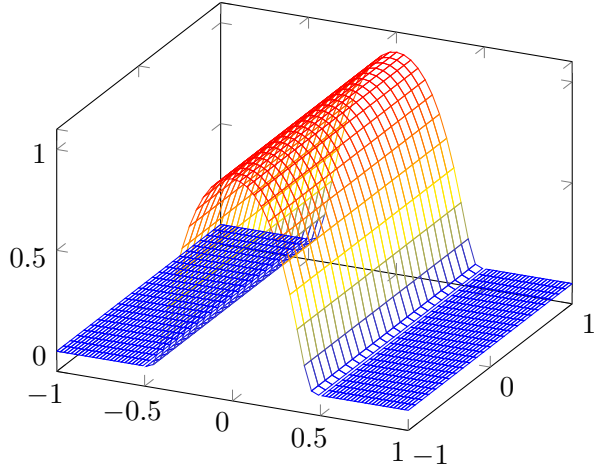
\begin{figure}
    \centering
    \begin{tikzpicture}
    \begin{axis}[xmin=-1,xmax=1,ymin=-1,ymax=1]
        \addplot3[mesh, domain=-0.49:0.49,y domain=-1:1] {exp(1-(1/(1-(4*x^2))))};
        \addplot3[mesh, domain=0.51:1,y domain=-1:1] {0};
        \addplot3[mesh, domain=-1:-0.51,y domain=-1:1] {0};
    \end{axis}
\end{tikzpicture}
    \caption{Ground truth function $f(\boldsymbol{x})$}
    \label{fig:ground_truth}
\end{figure}

\begin{figure}
    \centering
    \begin{tikzpicture}[scale=1.5]
        \draw[->, line width=1pt] (-1.5,0) -- (1.5,0);
        \draw[->, line width=1pt] (0,-1.5) -- (0,1.5);
        
        \draw[ dashed, line width=0.5pt] (0.5,-1.5) -- (0.5,1.5);
        \draw[ dashed, line width=0.5pt] (-0.5,-1.5) -- (-0.5,1.5);
        
        \tikzset{
            cross/.pic = {
            \draw[rotate = 45] (-#1,0) -- (#1,0);
            \draw[rotate = 45] (0,-#1) -- (0, #1);
            }
        }
        
        \node at (-1,-1.7) {$f=0$};
        \node at (0,-1.7) {$f=$ bump};
        \node at (1,-1.7) {$f=0$};

        \draw (1,0) pic[rotate=0] {cross=0.2} node[above right] {$b_1$};
        \draw (0,1) pic[rotate=0] {cross=0.2} node[above right] {$a_1$};
        \draw (-1,0) pic[rotate=0] {cross=0.2} node[above right] {$b_2$};
        \draw (0,-1) pic[rotate=0] {cross=0.2} node[above right] {$a_2$};
    \end{tikzpicture}
    \caption{Critical points for Counterexample.}
    \label{fig:critical_points}
\end{figure}
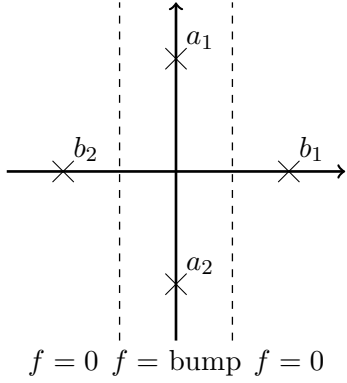

In this section, we will construct a counterexample showing that FC-BNNs are not capable of universally approximating a function living on $[-1,1]^2$, when using a single hidden layer and continuous input domains. The aim was to extend the proof of the previous section with one hidden layer FC-BNN to inputs being continuous valued. Our research indicates that it is not sufficient to use a single hidden layer FC-BNN to achieve universal approximation in certain functions of this type. This shows that the input domain has an impact on the universal approximation properties of FC-BNNs and interacts with the number of hidden layers used in such a network. The counterexample is based on a continuous bump-function, for which we select input points lying on parallel diagonals. This way, we show that the error of the approximating single hidden layer FC-BNN is at least $0.5$ and cannot be further dropped.

Focusing on a two-dimensional input space, we consider a ground truth function $f : [-1,1]^2 \to \Rbb$ which fulfills $f(0,x_2) = 1$ and $f(-1,x_2) = f(1,x_2) = 0$ for all $x_2 \in [-1,1]$. 
As an example, we consider a continuous and even differentiable function $f$ defined by
\begin{equation}\label{eq:gt_funct}
	f: \boldsymbol{x} = (x_1, x_2)^\top \mapsto \begin{cases}
	e^{1-\frac{1}{1-4 (x_1)^2}} & ,\,|x_1| < 0.5 \\
	0 & ,\,|x_1| \geq 0.5
	\end{cases}.
\end{equation}
Note that the function $f$ is a non-negative bump function attaining its maximum at $(0,0)$ as depicted in Figure~\ref{fig:ground_truth}.

In the following, we prove there is no single-hidden-layer fully connected BNN structure that can approximate $f(x_1,x_2)$ as defined above. In doing so, we first set up a Lemma, which delivers an equality between points lying on the axes of the considered bivariate grid. 

\begin{lemma}\label{lem:critical_property_easy}
    Let $\Ocal_{\BNN}^{h =1}$ be any fully connected single-hidden-layer BNN with activation function $\eqref{eq:activation_fct}$ and MAC operator as in $(\ref{MACoperator})$ aiming to approximate $f$ as given in \eqref{eq:gt_funct}.
    We define, $\boldsymbol{a}_1:= (0,1)^\top$, $\boldsymbol{a}_2:= (0,-1)^\top$, $\boldsymbol{b}_1:= (1,0)^\top$, $\boldsymbol{b}_2:= (-1,0)^\top$, as depicted in Figure~\ref{fig:critical_points}.
    For this BNN, the property 
    \begin{equation}\label{eq:critical_property_easy}
    \Ocal_{\BNN}^{h=1}(\boldsymbol{b}_1) + \Ocal_{\BNN}^{h=1}(\boldsymbol{b}_2)  = \Ocal_{\BNN}^{h=1}(\boldsymbol{a}_1) + \Ocal_{\BNN}^{h=1}(\boldsymbol{a}_2) 
    \end{equation}
    holds.
\end{lemma}

\begin{proof}
    We prove that~\eqref{eq:critical_property_easy} holds for each neuron output $o^1_{j}$ with $o^1_j(x_1,x_2) := \sigma_{b^1_j}( w^1_{1,j} x_1 + w^1_{2,j}x_2 )$.
    As a consequence, it then also holds for the linear combination $\Ocal^{h=1}_{\BNN}(\boldsymbol{x}) = \sum_j w^2_{j} \cdot o^1_j(\boldsymbol{x})$.
    To prove~\eqref{eq:critical_property_easy}, we consider all different cases for the two neuron weights $w^1_{1,j}$ and $w^1_{2,j}$:
    
    \textbf{Case~1:} $w^1_{1,j}$ and $w^1_{2,j}$ are both $+1$.
    In this case, the function value of $f_j$ is equal for points that lie on a straight line with normal vector $(1,1)^\top$, i.e., $v_1 := f_j(\boldsymbol{a}_1) = f_j(\boldsymbol{b}_1)$, $v_2 := f_j(\boldsymbol{b}_2) = f_j(\boldsymbol{a}_2)$.
    By inserting these values in Equation~\eqref{eq:critical_property_easy}, we obtain $v_1 + v_2 = v_1 + v_2$.
    Therefore, Equation~\eqref{eq:critical_property_easy} holds in this case.
    
    \textbf{Case~2 - Case~4:} $(w^1_{1,j}, w^1_{2,j})^\top$ are $(1,-1)^\top$, $(-1,1)^\top$ or $(-1,-1)^\top$.
    As in Case~1, the function values are the same if they are on a straight line with normal vector $(w^1_{1,j}, w^1_{2,j})$.
    Due to the symmetry of Equation~\eqref{eq:critical_property_easy}, it still holds in these cases.
\end{proof}

In the second part of the counterexample, we show that the function $f$ has an approximation lower bound of $1/2$ and does not admit universal approximation properties by a single hidden layer FC-BNN. Note that any function, fulfilling the condition stated in~\eqref{eq:critical_property_easy} cannot be universally approximated by a single hidden layer FC-BNN. 

\begin{lemma}\label{lem:no_approx_easy}
    Let $\Ocal^{h=1}_{\BNN}$ be any fully connected single-hidden-layer BNN with activation function $\eqref{eq:activation_fct}$ and MAC operator as in $\eqref{MACoperator}$ aiming to approximate $f$ as given in~\eqref{eq:gt_funct}. Then 
    \begin{equation}\label{eq:no_approx_easy}
        \sup_{\boldsymbol{x} \in \Rbb} \left|\Ocal_{\BNN}^{h=1}(\boldsymbol{x}) - f(\boldsymbol{x})\right| \geq \frac{1}{2}
    \end{equation}
\end{lemma}

\begin{proof}
    The result of the left hand side of~\eqref{eq:no_approx_easy} is lower bounded by
    \begin{equation}\label{eq:def_M}
        \begin{aligned}
            &\sup_{\boldsymbol{x} \in \Rbb} \left|\Ocal_{\BNN}^{h=1}(\boldsymbol{x}) - f(\boldsymbol{x})\right| \\&\geq \max_{\boldsymbol{x}\in \{\boldsymbol{a}_1,\textbf{a}_2,\boldsymbol{b}_1,\boldsymbol{b}_2\} } \left|\Ocal^{h=1}_{\BNN}(\boldsymbol{x}) - f(\boldsymbol{x})\right| =: M.  
        \end{aligned}
    \end{equation}
    More specifically $M$ is the maximum of the four values $|1-\Ocal^{h=1}_{\BNN}(\boldsymbol{a}_1)|$, $|1-\Ocal^{h=1}_{\BNN}(\boldsymbol{a}_2)|$, $|\Ocal^{h=1}_{\BNN}(\boldsymbol{b}_1)|$, and $|\Ocal^{h=1}_{\BNN}(\boldsymbol{b}_2)|$. We obtain these values by inserting $\boldsymbol{a}_1$, $\boldsymbol{a}_2$, $\boldsymbol{b}_1$, $\boldsymbol{b}_2$ into the left hand side of~\eqref{eq:no_approx_easy} and use the definition of the bump function $f(\boldsymbol{x})$ in~\eqref{eq:gt_funct}. 
    
    Due to Lemma~\ref{lem:critical_property_easy}, i.e., the condition in~\eqref{eq:critical_property_easy} holds, we reformulate it to
    \begin{equation}\label{eq:M>=0.5}
        \begin{aligned}
            2 = &(1-\Ocal_{\BNN}^{h=1}(\boldsymbol{a}_1)) + (1- \Ocal_{\BNN}^{h=1}(\boldsymbol{a}_2)) \\&+ \Ocal_{\BNN}^{h=1}(\boldsymbol{b}_1) + \Ocal_{\BNN}^{h=1}(\boldsymbol{b}_2)\\
            \leq& \;4M,
        \end{aligned}
    \end{equation}
    where the inequality is due to the definition of $M$.
    It follows that $M\geq \frac{1}{2}$.
    Using~\eqref{eq:def_M}~and~\eqref{eq:M>=0.5}, we conclude the result.
\end{proof}

\begin{remark}
    In fact, this negative result is also extendable to a tertiary input spaces.
    Consider a two dimensional tertiary input space, i.e., $X=\{-1,0,1\}^2$.
    We define the ground truth function by $f(0,x_2) = 1$ and $f(-1,x_2) = f(1,x_2) = 0$ for all $x_2 \in \{-1,0,1\}$.
    The proofs of Lemma~\ref{lem:critical_property_easy} and Lemma~\ref{lem:no_approx_easy} consider only the values of $f$ at $\boldsymbol{a}_1, \boldsymbol{a}_2, \boldsymbol{b}_1, \boldsymbol{b}_2$ and are therefore still valid for this specific case. 
\end{remark}

\section{Two Hidden Layer, FC-BNN, Real Inputs}
\label{sec:2_fc_real}

A natural question extended from Section~\ref{sec:1_fc_real} is whether FC-BNNs can admit UA properties with two hidden layers when the input domain is continuous. We provide a positive answer to the question by a constructive proof in the following theorem, indicating that two layers are sufficient to admit UA properties as long as the function $f$ has Lipschitz continuity.

\begin{figure}
    \centering
    \begin{tikzpicture}[scale=1.5]
        \draw[->, line width=1pt] (-0.3,0) -- (1.3,0);
        \draw[->, line width=1pt] (0,-0.3) -- (0,1.3);
        
        
        \draw[blue] (0,0) rectangle (1,1);
        
        \draw[red, dashed] (1.5,0.5) -- (0.5,1.5);
        \draw[red] (1.6,-0.1) node[below right] {4} -- (.15,1.35);
        \draw[red] (1.35,-0.35)node[below right] {3} -- (-0.1,1.1);
        \draw[red] (1.1,-0.6) node[below right] {2} -- (-0.35,.85);
        \draw[red] (.85,-.85) node[below right] {1} -- (-.6,.6);
        
        \draw[red, dashed] (-.5,0.5) -- (0.5,1.5);
        \draw[red] (-.6,-.1) node[below left] {4} -- (0.85,1.35);
        \draw[red] (-.35,-.35) node[below left] {3} -- (1.1,1.1);
        \draw[red] (-.1,-.6) node[below left] {2} -- (1.35,0.85);
        \draw[red] (.15,-.85) node[below left] {1} -- (1.6,.6);
        
        \draw[->, red] (1.2,1.2) to (1.5,1.5) node[above right] {$(1,1)$};
        \draw[->, red] (-.2,1.2) to (-.5,1.5) node[above left] {$(-1,1)$};
        
        \node[red] at (0.75,0.5) {A};

    \end{tikzpicture}
    \caption{First Layer for Theorem~\ref{thm:2_hidden_FC_bnn}.}
    \label{fig:first_layer}
\end{figure}
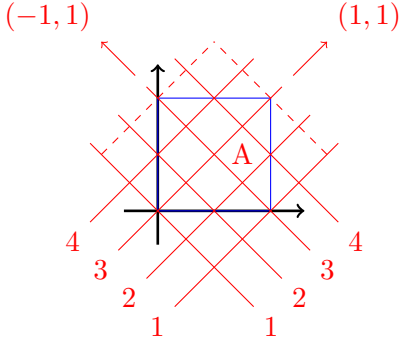

\begin{theorem}\label{thm:2_hidden_FC_bnn}
    Let $X \subseteq \mathbb{R}^d$ be some arbitrary dimensional compact input space and let $f: X \to \Rbb$ be the ground-truth function.
    Let $f$ be Lipschitz continuous, i.e., we have $\lambda\in \mathbb{R}$ such that for all $\boldsymbol{x}, \boldsymbol{y} \in X$ the property $|f(\boldsymbol{x})-f(\boldsymbol{y})| \leq \lambda \cdot \norm{ \boldsymbol{x}-\boldsymbol{y}}$ holds.
    For all $\eps>0$ there exists some fully connected BNN with two hidden layers $\Ocal_{\BNN}^{h=2}$ such that
    \begin{equation}\label{eq:ua_thm3}
        \sup_{\boldsymbol{x} \in X} \left|\mathcal{O}^{h=2}_{\BNN}(\boldsymbol{x}) - f(\boldsymbol{x})\right| < \eps
    \end{equation}{}
    holds.
\end{theorem}

\begin{proof}
    We use a similar construction as in \cite{spallanzani/etal/2019} for the first layer and Theorem~\ref{thm:bin_input} for the second layer.
    
    \textbf{First Layer:}
    At first, we choose $d$ linear independent vectors with entries only in $\{1,-1\}$, e.g., $\boldsymbol{v}^1 = (1,\dots,1)$, $\boldsymbol{v}^2=(-1,1, \dots, 1)$, $\boldsymbol{v}^3 = (1,-1,1, \dots, 1)$, $\dots$, $\boldsymbol{v}^d=(1, \dots, 1, -1,1) \in \Rbb^{d}$.
    Given a length $s$ of the edges and a base point $p \in \Rbb^d$, these vectors span a parallelotope $P^s_p := P^s_p(\boldsymbol{v}^1, \dots, \boldsymbol{v}^d)$.
    
    By the Lipschitz continuity we know that if $\norm{\boldsymbol{x} - \boldsymbol{y}} < \frac{\eps}{\lambda}$ then $|f(\boldsymbol{x})-f(\boldsymbol{y})|<\eps$.
    We choose $s \in \Rbb>0$ small enough such that $\diam(P^s_p(\boldsymbol{v}^1, \dots, \boldsymbol{v}^d)) < \frac{\eps}{\lambda}$.
    As a result, if there exists some base point $p$ such that $\boldsymbol{x},\boldsymbol{y} \in P^s_p$, then $|f(\boldsymbol{x})-f(\boldsymbol{y})| <\eps$.
    
    In the following, we segment the input space into parallelotopes of the form $P^s_p$ and translate this into the language of neurons.
    For each vector $\boldsymbol{v}^i$, we find the minimal and maximal values of $X$ in that direction by using the dot product. 
    More specifically, we define $m^i_1 := \min_{\boldsymbol{x}}( \boldsymbol{x}^\top \boldsymbol{v}^i)$ and $m^i_2 := \max_{\boldsymbol{x}}( \boldsymbol{x}^\top \cdot \boldsymbol{v}^i)$.
    We choose base points 
    \begin{equation}\label{eq:base_points}
        p(\xi_1, \dots, \xi_d) = \sum_{i=1}^d (\xi_i \cdot s + m^i_1) \cdot \boldsymbol{v}^i
    \end{equation}
    with $\xi_i \in [0,m^i_2-m^i_1] \cap \Zbb$ for $i=1,2,\ldots,d$.
    The parallelotopes $P^s_{p(\xi_1, \dots, \xi_d)}$ derived by those base points partition the whole input space $X$ as depicted in Figure~\ref{fig:first_layer}.
    
    We translate the partition of parallelotopes into a set of neurons on the first layer, i.e., each output of the first layer uniquely determines the parallelotope of the input.
    For each vector $\boldsymbol{v}^i$ we put neurons on the first layer with weights $\boldsymbol{v}^i$ and thresholds $\xi_i \cdot s + m^i_1$ with $\xi_i \in [0,m^i_2-m^i_1] \cap \Zbb$.
    The output of these neurons uniquely determines the factors in~\eqref{eq:base_points} and therefore also determines the corresponding parallelotope.
    
    An example of our construction is depicted in Figure~\ref{fig:first_layer} for $d = 2$ dimensions of the input space.
    The blue square is the input space $X$ and the red lines illustrate the neurons on the first layer as follows.
    In particular, the neurons of the vector $\boldsymbol{v}^i$ are represented by the orthogonal lines.
    If the input is in some parallelotope $A$ for example, then the first three neurons for the vector $(1,1)$ have the output $+1$ because the parallelotope is on the side of the vector direction and the fourth neuron has the output $-1$ as the parallelotope is on the opposite direction.
    For the vector $(-1,1)$ the first two neurons have an output of $+1$ and the last two have an output of $-1$.
    
    \textbf{Second Layer:}
    We use Theorem~\ref{thm:bin_input} to create the second layer. 
    For each parallelotope $P^s_{p(\xi_1, \dots, \xi_d)}$ the inputs in that parallelotope have a specific neuron output after the first layer.
    We use the construction from the proof of Theorem~\ref{thm:bin_input} to assign to it to a specific value.
    We can use any value $f(\boldsymbol{x})$ inside that parallelotope $ \boldsymbol{x} \in P^s_{p(\xi_1, \dots, \xi_d)}$, e.g., $f(p(\xi_1, \dots, \xi_d))$.
    By construction of the parallelotope and the Lipschitz continuity of $f$, the difference between this value and any value $f(\boldsymbol{y}),~ \boldsymbol{y}\in P^s_{p(\xi_1, \dots, \xi_d)}$ is at most $\eps$.
    More formally,
    \begin{equation}
        \begin{aligned}
            \left|\Ocal^{h=2}_{\BNN}(\boldsymbol{x}) - f(\boldsymbol{x})\right| &= |{f(\boldsymbol{p}) - f(\boldsymbol{x})}|
            \\&\leq \lambda \cdot \norm{\boldsymbol{p}-\boldsymbol{x}} < \lambda \cdot \frac{\eps}{\lambda} = \eps
        \end{aligned}
    \end{equation}
    for $\boldsymbol{p} = p(\xi_1, \dots, \xi_d)$ such that $\boldsymbol{x} \in P^s_{p(\xi_1, \dots, \xi_d)}$.
\end{proof}

\section{Remove Real Weights From Output Layer}
\label{sec:removal-real-weights}
\begin{figure}
    \centering
    \newcommand{\circlediam}{0.3cm}
        \begin{tikzpicture}[xscale=0.85]


        \node[anchor=north, text width=2cm] at (0,3) {$\dots$};
        \node[anchor=north, text width=2cm] at (0,1) {$\dots$};
        \node[anchor=north, text width=2cm] at (0,-1) {$\dots$};

        \node[minimum size=1cm,draw,circle] (h3) at (0,3) {$o^{L-1}_1$};
        \node[minimum size=1cm,draw,circle] (h2) at (0,1) {$o^{L-1}_2$};
        \node[minimum size=1cm,draw,circle] (h5) at (0,-1) {$o^{L-1}_{n_{L-1}}$};

        \node[anchor=north, text width=2cm, align=center] at (0,-1.5) {hidden layer\\($\ell = {L-1}$)};
        \node[anchor=north, text width=2cm, align=center] at (3,-1.5) {output layer\\($\ell = L$)};


        \node at ($(h2)!.5!(h5)$) {\vdots};

        \coordinate (D) at ($(h3)!.5!(h5)$);
        \ExtractCoordinate{D};
        \node[minimum width=1cm,draw,circle] (o1) at (3,\YCoord) {$\sum$};

        \node[] (out) at (7,\YCoord) {$\mathcal{O}^{}_{\BNN}(\mathbf{x})$};




        \draw[->] (h2) -- (o1);
        \draw[->] (h3) -- node[above, anchor=west] {\textcolor{red}{$w^L_{k} \in \{\pm1\}$}} (o1);
        \draw[->] (h5) -- (o1);
        \draw[->] (o1) -- node[above]{\textcolor{red}{$\cdot \alpha \in \Qbb$}} (out);

    \end{tikzpicture}
    \caption{Last layers of a BNN with binarized weights in layer $\ell=L$ and scaling factor $\in \Rbb$ in the output}
    \label{fig:bnn_model_outputscale}
\end{figure}
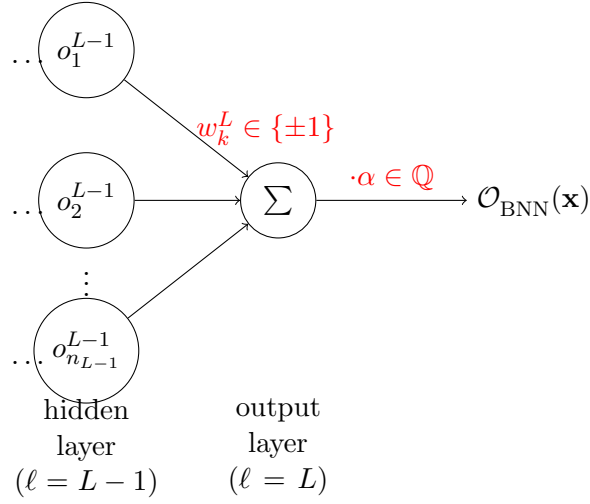

In this section we shortly discuss that our results can be transferred to neural networks that have only binary weights (even in the output layer) when introducing a scaling factor, i.e., an additional factor that is introduced after the output of the neural network. We show that the relaxation of binarizing the output weights does not have severe effect on the output accuracy and simplifies BNNs even further while almost maintaining output accuracy. Our proof is mainly based on the mathematical fact that $\mathbb{Q}$ lies dense in $\Rbb$. Binarized output weights then lead to BNNs that are fully binarized, i.e., every weight is restricted to $\{\pm 1\}$, not only the weights directing to the hidden layers.  

\begin{theorem}\label{thm:real_weights_to_scaling}
    Let $X\subseteq \Rbb^d$ be arbitrary and consider activation functions of the form $\eqref{eq:activation_fct}$. Any binarized neural network $\Ocal_{\BNN}$ can be approximated by a binarized neural network $\widetilde\Ocal_{\BNN}$ with the same number of hidden layers, which has only binary weights on the output layer and a scaling factor $\alpha \in \Qbb$ after the output layer, as depicted in Figure~\ref{fig:bnn_model_outputscale}.
    More specifically, for any $\eps>0$ there exists some $\widetilde\Ocal_{\BNN}$ as above, such that 
    \begin{equation}\label{eq:thm:real_weights_to_scaling}
        \sup_{\boldsymbol{x} \in X} \left|\Ocal_{\BNN}(\boldsymbol{x}) - \widetilde\Ocal_{\BNN}(\boldsymbol{x})\right| < \eps
    \end{equation}
    holds.
    If $\Ocal_{\BNN}$ is fully connected, then $\widetilde\Ocal_{\BNN}$ can be chosen to be also fully connected.
\end{theorem}

\begin{proof}
    We denote by $\boldsymbol W^L = (w^L_1, \dots, w^L_{n_{L-1}})^\top \in \Rbb^{n_{L-1}}$ the weights of the output layer of $\Ocal_{\BNN}$.
    The proof is divided in two steps. 
    At first we show that there exists a binarized neural network $\widehat\Ocal_{\BNN}$, which has only rational valued weights $\boldsymbol{ \widehat{W}}^L = (\widehat w^L_1, \dots, \widehat w^L_{n_{L-1}})^\top \in \Qbb^{n_{L-1}}$ for the output layer, such that $\sup_{\boldsymbol{x} \in X} \left|\widehat\Ocal_{\BNN}(\boldsymbol{x}) - \Ocal_{\BNN}(\boldsymbol{x})\right| < \eps$.
    In a second step we define a BNN $\widetilde \Ocal_{\BNN}$ with only binarized weights $\boldsymbol{\widetilde{W}}^L \in \{\pm1\}^{n_{L-1}}$ in the output layer, such that $\widehat\Ocal_{\BNN}(\boldsymbol{x}) = \widetilde\Ocal_{\BNN}(\boldsymbol{x})$ for all $\boldsymbol{x}\in X$.
    
    \textbf{Step 1:} Construction of $\widehat\Ocal_{\BNN}$.
    Since the activation function, as defined in Equation~\eqref{eq:activation_fct}, is binarized and there are only finitely many neurons on the $(L-1)$-th layer, we know that the image of the function $o^{L-1} : X \to \Rbb^{n_{L-1}}$ which computes the output of the last hidden layer is a finite set.
    We denote by $K := \text{Im}(o^{L-1}) = \{ \pm 1 \} \subset \Rbb^{n_{L-1}}$ the finite image of the last hidden layer.
    Moreover, since $\Qbb$ is dense in $\Rbb$, we can choose rational values $\widehat{w}^L_i \in \Qbb$ for all $i \in \{1,\dots,n_{L-1}\}$ such that $|w^L_i - \widehat w^L_i| < \frac{\eps}{\max(K)\cdot n_{L-1}}$.
    The BNN $\widehat\Ocal_{\BNN}$ is obtained by replacing the weights in the output layer of $\Ocal_{\BNN}$ by $\boldsymbol{\widehat{W}}^L = (\widehat w^L_1, \dots, \widehat w^L_{n_{L-1}})^\top$.
    This leads to 
    \begin{align}
        &\left|\Ocal_{\BNN}(\boldsymbol{x}) - \widehat\Ocal_{\BNN}(\boldsymbol{x})\right| \notag
        \\&= |{o^{L-1}(\boldsymbol\phi^{L-1}(\boldsymbol{x}))^\top \cdot \boldsymbol W^L - o^{L-1}(\boldsymbol \phi^{L-1}(\boldsymbol{x}))^\top \cdot \widehat{\boldsymbol{W}}^L}| \notag
        \\&\leq \sum_{i=1}^{n_{L-1}} |o^{L-1}(\boldsymbol\phi^{L-1}(\boldsymbol{x})) \cdot (w_i^L - \widehat w_i^L)| \notag
        \\&\leq \sum_{i=1}^{n_{L-1}} \max(K) \cdot  |w_i^L-\widehat w_i^L| \notag
        \\&< \sum_{i=1}^{n_{L-1}} \max(K) \cdot  \frac{\eps}{\max(K)\cdot n_{L-1}} = \eps
    \end{align}
    for any $\boldsymbol{x} \in X$.
    More specifically, $\widehat\Ocal_{\BNN}$ approximates $\Ocal_{\BNN}$ as required.
    
    \textbf{Step 2:} Construction of $\widetilde\Ocal_{\BNN}$.
    Since the weights $\widehat w^L_i$ of $\widehat\Ocal_{\BNN}$ are rational numbers, we can write them as $\widehat w^L_i = \frac{z_i}{q_i}$ with $z_i\in \Zbb$ and $q_i \in  \Nbb $. 
    We bring all $\widehat w^L_i$ to the same denominator, with 
    \begin{equation*}
        \widehat w^L_i = \frac{z_i\prod_{j\neq i} q_j}{\prod_{j} q_j}.
    \end{equation*}
    We define $\widehat{z}_i:= z_i\prod_{j\neq i} q_j$ and $q:=\prod_{j} q_j$, i.e., $\widehat w^L_i = \frac{\widehat z_i}{q}$, and choose $\alpha := \frac{1}{q}$ as the scaling factor after the output.
    To obtain $\widetilde\Ocal_{\BNN}$, we replace the $i$-th neuron on the $(L-1)$-th layer in $\widehat\Ocal_{\BNN}$ by $|\widehat{z}_i|$ many copies of the neuron for all $1 \le i \le n_{L-1}$ and set the weights from each of these neurons to the output neuron to $1$ if $\widehat{z}_i\geq 0$ and $-1$ if $\widehat{z}_i < 0$.
    The resulting BNN $\widetilde\Ocal_{\BNN}$ has $\sum_{i=1}^{n_{L-1}}|\widehat{z}_i|$ neurons on the last hidden layer and the weights $\boldsymbol{\widetilde{W}}^L$ in the output layer are all binarized.
    With the scaling factor $\alpha = \frac{1}{q}$, $\widetilde\Ocal_{\BNN}(x)$ coincides with the BNN $\widehat\Ocal_{\BNN}$, i.e., $\widetilde\Ocal_{\BNN}(\boldsymbol{x}) = \widehat\Ocal_{\BNN}(\boldsymbol{x})$ for all $ \boldsymbol{x} \in X$.

    Please note that the transformation from $\Ocal_{\BNN}$ to $\widehat\Ocal_{\BNN}$ as well as to $\widetilde\Ocal_{\BNN}$ preserves the property to be fully connected since we only exchange weights and duplicate or remove neurons together with their weights.
\end{proof}

\begin{remark}
In Theorem \ref{thm:real_weights_to_scaling}, the assumption of considering activation functions of the type as in $\eqref{eq:activation_fct}$ can be relaxed. Instead, we can consider any activation function that is suitable for a BNN, but restrict the attention to pre-last-hidden layer activation functions $\sigma^{L-1}$ being of the type: 
\begin{align}\label{alternativeCond}
    \sup\limits_{t} |\sigma^{L-1}(t)| < \infty.
\end{align}
This will immediately result to $\sup\limits_{\boldsymbol{x}} |o^{L-1}(\boldsymbol{\phi^{L-1}}(\boldsymbol{x}))| < \infty$ leading to the desired result. Note that the feature input space does not have to be compact. However, the latter scenario is also covered by the alternative condition~\eqref{alternativeCond}. This is because assuming that the input space is compact and $\sigma^{L-1}$ is continuous, the condition automatically results. Hence, condition \eqref{alternativeCond} is wider. 
\end{remark}

Moreover every neural network that has an $\Rbb$-valued scaling factor after the neuron output can be transformed into a neural network without a scaling factor, by multiplying the scaling factor with every weight of the output layer.
As a consequence, the weights of the output layer might become real valued.

Based on the results of this section, the counterexample from Section~\ref{sec:1_fc_real} and the UA theorems from Section~\ref{sec:1_fc_bin} and Section~\ref{sec:2_fc_real} are also valid for neural networks with only binary weights (even on the output layer) and a scaling factor.

\section{Conclusion}
\label{sec:conclusion}
Binarized neural networks (BNN) with binarized weights and activations offer significant reductions in memory demand and inference latency, which makes them a viable option for resource constrained systems. In this work, we focused on the theoretical foundations of fully connected BNNs, regarding the universal approximation (UA) property of BNNs. 

We first proved that single-hidden-layer fully connected BNNs with binarized inputs admits the universal approximation (UA) property. We then extended this setting to inputs in $\Rbb$ and showed by a counterexample that one hidden layer is not sufficient for admitting the UA property of FC-BNNs. Finally, we extended our considerations to two-hidden-layer FC-BNNs and proved that the UA property holds in this case. Therefore, FC-BNNs with real inputs require at least two hidden layers for admitting UA properties. In the future, we will use the insights and analyses in this work to optimize BNN structures and extend the properties to statistical learning issues in classification and regression such as consistency results.

\section*{Acknowledgement}
This work has been supported by Deutsche Forschungsgemeinschaft (DFG), as part of Sus-Aware (Project No. 398602212), and the collaborative research center SFB876, subproject A1. The work of Burim Ramosaj is supported by the Ministry of Culture and Science of the state of NRW (MKW NRW) through the research grant \textit{KI-Starter}. 

\bibliographystyle{plain}
\bibliography{citations}

\end{document}